%% file: main.tex
\documentclass[11pt]{article}

\usepackage[margin=1in]{geometry}
\usepackage{microtype}
\usepackage{mathtools}
\usepackage{amssymb}
\usepackage{amsthm}
\usepackage{mleftright}
\usepackage[dvipsnames]{xcolor}
\usepackage{graphicx}
\usepackage[round]{natbib}
\usepackage{tikz}
\usepackage{pgfplots}
\usepackage[colorlinks=true,allcolors=MidnightBlue]{hyperref}
\usepgfplotslibrary{groupplots}
\pgfplotsset{compat=1.18}

\title{Effects of sparsity and superposition on loss in simple autoencoders}
\author{%
Mriganka Basu Roy Chowdhury\thanks{Equal contribution.}\\
Department of Statistics, UC Berkeley\\
\texttt{mbrc12@gmail.com}
\and
Eric McLaughlin Weiner\footnotemark[1]\\
Department of Materials Science, UC Berkeley\\
\texttt{eric\_weiner@berkeley.edu}
}
\date{}

\theoremstyle{plain}
\newtheorem{theorem}{Theorem}[section]
\newtheorem{lemma}[theorem]{Lemma}

\theoremstyle{definition}

\theoremstyle{remark}
\newtheorem{remark}[theorem]{Remark}

\def\les{\lesssim}
\def\gtr{\gtrsim}
\def\om{{\odot m}}
\def\al{\alpha}

\def\op{\mathrm{op}}
\newcommand{\df}{\coloneq}

\newcommand{\N}{\mathbb{N}}
\newcommand{\R}{\mathbb{R}}
\newcommand{\enc}{\mathsf{enc}}
\newcommand{\dec}{\mathsf{dec}}
\newcommand{\cL}{\mathcal{L}}
\def\cF{\mathcal{F}}
\def\L{\widetilde{\cL}}
\newcommand{\norm}[1]{\mleft\|#1\mright\|}
\newcommand{\E}{\mathbb{E}}

\newcommand{\Cov}{\mathrm{Cov}}
\newcommand{\Var}{\mathrm{Var}}
\renewcommand{\d}{\mathrm{d}}
\renewcommand{\Box}[1]{\mleft[ #1 \mright]}
\newcommand{\Rnd}[1]{\mleft( #1 \mright)}
\newcommand{\Ber}{\mathrm{Bernoulli}}
\def\Om{\Omega}
\newcommand{\T}[1]{\Theta\Rnd{#1}}

\newcommand{\lam}{\lambda}
\renewcommand{\L}{\tilde{\cL}}
\newcommand{\f}[2]{\frac{#1}{#2}}
\newcommand{\inn}[1]{\mleft\langle #1 \mright\rangle}

\newcommand{\tr}{\mathop{\mathrm{tr}}}
\newcommand{\rk}{\mathop{\mathrm{rank}}}
\newcommand{\Abs}[1]{\mleft| #1 \mright|}
\def\Sig{\Sigma}
\def\leq{\leqslant}
\def\geq{\geqslant}

\def\1{\mathbf{1}}

\begin{document}

\maketitle
\iffalse
\texttt{
One of the major difficulties in the mechanistic interpretability of neural networks is the occurrence of polysemanticity, which suggests that each neuron is typically responsible for multiple different tasks, impeding a clean interpretation of their function. The seminal paper of Elhage et al. (2022) [1] argues that this occurs due to superposition, a phenomenon where the neural network represents distinct features as non-orthogonal directions in a lower-dimensional space, a strategy that allows much greater compression of the data without sacrificing fidelity due to the feature sparsity of input vectors. [1] empirically validates these hypotheses in a rather natural and simple autoencoder with sparse inputs. The contribution of the present work is to analyze the mathematical basis for the occurrence and optimality of superposition, while rigorously corroborating some their findings. In particular, we provide upper and lower bounds for the L^2 reconstruction loss, tight in the very sparse regime, for power activation functions. A short list of interesting open problems is also included at the end.

}
\fi

\begin{abstract}%
One of the major difficulties in the mechanistic interpretability of neural networks is the
occurrence of \textit{polysemanticity}, which suggests that each neuron is typically responsible for
multiple different tasks, impeding a clean interpretation of their function. The seminal paper of
\cite{toy} argues that this occurs due to \textit{superposition}, a phenomenon where the neural
network represents distinct features as non-orthogonal directions in a lower-dimensional space, a
strategy that allows much greater compression of the data without sacrificing fidelity due to the
feature sparsity of input vectors. \cite{toy} empirically validates these hypotheses in a rather
natural and simple autoencoder with sparse inputs. The contribution of the present work is to
analyze the mathematical basis for the occurrence and optimality of superposition, while rigorously
corroborating some of their findings. In particular, we provide upper and lower bounds for the
$L^2$ reconstruction loss, tight in the very sparse regime, for power activation functions.
A short list of interesting open problems are also included at the end.
\end{abstract}

\section{Introduction}

Advances in the capabilities and impacts of large-scale neural networks have necessitated a deeper
understanding of their internal mechanisms, with the hope that these insights will lead to more
reliable and customizable models. The field of \textit{mechanistic interpretability} (see, for instance, \cite{mechint}) seeks to
uncover the detailed inner workings of trained neural networks, and although nascent, a vast array of
empirical findings have been reported. However, since neural networks learn directly from data,
without any bias towards human-interpretable features, it is often difficult to cleanly separate and
ascribe functionality to individual neurons or layers. A key challenge in this regard is the
presence of \textit{polysemantic} neurons, which respond to multiple different features, thereby
obscuring their role. The recent work of \cite{toy} proposes that this phenomenon arises due to the
neural network attempting to pack more features into a limited number of neurons, a phenomenon they
refer to as \textit{superposition}. A hypothesis for the effectiveness of this strategy is that
inputs are typically \textit{feature-sparse}, that is, each input typically exhibits only a few of
the many possible learned features. Consequently, this ``overloading'' does not significantly impact
fidelity. Various aspects of this phenomenon have been studied extensively, for instance, recovering
superposed features in \cite{bricken2023towards,sharkey2022taking}, or understanding their occurrence in large language
models in \cite{cunningham2023sparse}. The literature also follows a long line of research
on sparse coding, starting with \cite{olshausen1996emergence}, which seeks to decompose a signal into a sparse
combination of basis elements; for connections to autoencoders see \cite{rangamani2018sparse}. 
We refer the reader to the references in \cite{toy} (also \cite{super2, bricken2023towards}) for a more
comprehensive overview of prior work.

\cite{toy} proposes a simple one-layer autoencoder (with tied weights) to study the effects of
sparsity and superposition. The core purpose of the current paper is to attempt to rigorously understand
this simple model, and to provide theoretical guarantees on just how beneficial is superposition
in the presence of sparsity, and how nonlinearity aids the same. Without any further ado, we now
describe the model, which is almost identical to that of \cite{toy}, with some minor modifications
to aid analysis. We note that throughout this paper, we consider the ``equal-importance'' case as
the results are cleanest here; we strongly believe our techniques extend to the general case as
well.

\section{Model and results}

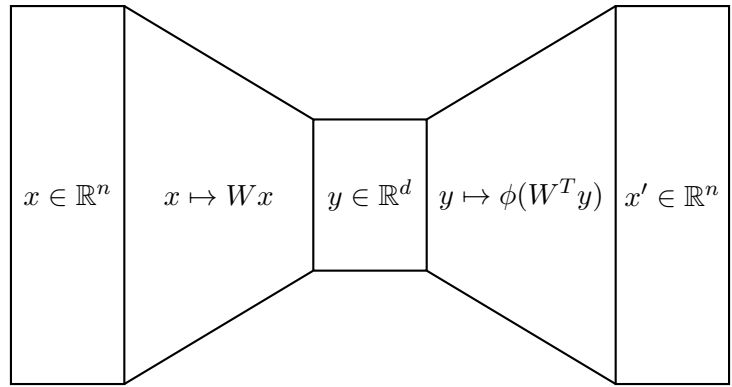
\begin{figure}[h]
    \centering
    \begin{tikzpicture}[scale=0.5]
        \draw[thick] (0,0) rectangle (3, 10) node[pos=0.5] {$x \in \R^n$};
        \draw[thick] (3, 10) -- (8, 7);
        \draw[thick] (3, 0) -- (8, 3);
        \draw (5.5, 5) node {$x \mapsto Wx$};
        \draw[thick] (8, 3) rectangle (11, 7) node[pos=0.5] {$y \in \R^d$};
        \draw[thick] (11, 7) -- (16, 10);
        \draw[thick] (11, 3) -- (16, 0);
        \draw (13.5, 5) node {$y \mapsto \phi(W^T y)$};
        \draw[thick] (16, 0) rectangle (19, 10) node[pos=0.5] {$x' \in \R^n$};
    \end{tikzpicture}
    \caption{The one-layer autoencoder we consider. $n$ is the input dimension
    and $d$ is the hidden dimension. $\phi$ is an activation function acting
    pointwise. The \textit{reconstruction} is $x' = x'(x)$.}
    \label{fig:model}
\end{figure}

Following \cite{toy}, we consider a simplified one-layer autoencoder, see Figure
\ref{fig:model}:
\begin{itemize}
    \item Input samples are $x \in \R^n$ drawn from a distribution defined below.
    \item The (tied-weights) encoder and decoder are chosen to be of the form
        \begin{align}
            \label{eq:enc} \enc(x) &= Wx, \quad x \in \R^n\\
            \label{eq:dec} \dec(y) &= \phi(W^T y), \quad y \in \R^d,
        \end{align}
        where $W \in \R^{d \times n}$ is a trainable weight matrix and $\phi : \R \to \R$ 
        is an activation function acting pointwise. This defines the reconstruction
        of $x$ as $x'(x) = \dec(\enc(x)) = \phi(W^T W x)$, on which we impose the \textit{squared-$L^2$ loss}
        $\E\norm{x - x'(x)}_2^2$.
    \item The activation function is chosen to be a power function
        \begin{align}
            \phi(t) = t^m, \quad m \geq 1 \text{ odd integer.} \nonumber
        \end{align}
    \item The inputs are $p$-sparse vectors, generated as follows:
        \begin{align}
             \label{eq:rogh}
             x &= (\xi_1 b_1, \xi_2 b_2, \ldots, \xi_n b_n) \in \R^n, \\
             \xi_1, \xi_2, \ldots, \xi_n& \sim \mu,\; \text{i.i.d., and}\\
             b_1, b_2, \ldots, b_n & \sim \Ber(p),\; \text{i.i.d.,}
        \end{align}
        where $\mu$ is a symmetric mean-zero distribution on $\R$ with all finite moments.
\end{itemize}

\begin{remark}
     \label{rmk:eivu}
    Note that we do not include a bias term in our model, which is aligned with our choice
    of the distribution $\mu$ being symmetric and mean-zero, as well as $\phi$ being an \textit{odd
    power}.
\end{remark}

Under the assumptions above we can write down the
population loss\footnote{We think of the number of samples as large enough to
allow an approximation of the sample loss by the population loss.} as
\begin{align}
    \label{eq:lossrel}
    \cL(W) &= \E \norm{\phi(W^T W x) - x}_2^2 \nonumber \\
           &= 
           \E\norm{x}_2^2 - \Rnd{2 \E\inn{x, \phi(W^T W x)} - \E\norm{\phi(W^T W
    x)}_2^2} \nonumber \\
           &= p n \mu_2 - \L(W),
\end{align}
where we define
\begin{align}
    \label{eq:loss}
    \L(W) &\df 2 \cdot \E\inn{x, \phi(W^T W x)} - \E\norm{\phi(W^T W x)}_2^2,
\end{align}
and the moments
\begin{align}
    \mu_k \df \int \xi^k \mu(\d \xi), \quad k = 1, 2, \ldots.
\end{align}
Minimizing $\cL(W)$ is equivalent to maximizing $\L(W)$; we choose this version
to eliminate the part of $\cL$ that we cannot control. 

The key hypothesis of \cite{toy} is that the nonlinearity $\phi$ allows the autoencoder to
\textit{superpose} features, which translates to a $W$ with non-orthogonal columns. To contrast our
results and to form a baseline for comparison, let us quickly inspect the case when $W$ is indeed
unsuperposed, that is $A \df W^T W$ is diagonal (we will use this notation repeatedly in the rest of
this article). Note that since $\rk A \leq d$, at most $d$ of the diagonal entries can be nonzero. In this case we have
\begin{align}
     \label{eq:zqvn}
     \L(W) &= 2 \cdot \E\inn{x, \phi(A x)} - \E\norm{\phi(A x)}_2^2 \nonumber \\
           &= \sum_{i=1}^n \Rnd{2 \cdot \E[x_i \phi(A_{ii} x_i)] - \E[\phi(A_{ii} x_i)^2]} \nonumber \\
           &= p \cdot \sum_{i=1}^n \Rnd{2\cdot \mu_{m+1}\cdot A_{ii}^m - \mu_{2m} \cdot A_{ii}^{2m}} \nonumber \\
           &= O(pd),
\end{align}
by noting that $2c_1 t^m - c_2 t^{2m} \leq \f{c_1^2}{c_2} = O(1)$ for all $t$, where the $O(\cdot)$
notation hides constants depending only on $m$ and $\mu$. Further, since the optimization problems
are independent across $i$, this upper bounded is indeed attainable, i.e.,
\begin{align}
     \label{eq:cgot}
     \sup_{\text{unsuperposed}\; W} \L(W) = \T{pd}.
\end{align}

It is also illuminating to consider the \textit{linear case}, where $\phi(x) = x$. This
choice simplifies \eqref{eq:loss} to
\begin{align*}
    \L(W) &= 2 \cdot \E\inn{x, W^T W x} - \E\norm{W^T W x}_2^2 \\
          &= p \mu_2 \cdot \tr(2A - A^2).
\end{align*}
Since $\rk A = \rk W \leq d$, one may
observe that
\[
    \tr(2A - A^2) = \sum_{i=1}^n (2\lam_i - \lam_i^2) \leq d,
\]
regardless of the choice of eigenvalues. Therefore,
\textbf{in both the linear case and the unsuperposed case, the optimal loss is of order $pd$.}
With this baseline in mind, we now state our main result.

\begin{theorem}
    \label{thm:main}
    There are some constants $C_1, C_2, C_3, d_0$ depending only on $m$ and $\mu$ such that
    the following holds for all $d > d_0$ and $n > d^m$:
    \begin{enumerate}
        \item We have the bounds \[
            C_1 \min(pd^m, dp^{1/m}) \leq \sup_W \L(W) \leq C_2 pd^m.
        \]
        \item If in addition for some $K > 1$ we have that $\mu$ is $(K^{-1}, K)$-strongly log-concave, that is, it has density $\propto e^{-v(x)}$ where $v''(x) \in [K^{-1}, K]$ for all $x$, then
        \[
            \sup_W \L(W) \leq C_3 d,
        \]
        so that in combination with the above we have
        \[
            \min(pd^m, dp^{1/m}) \les \sup_W \L(W) \les \min(pd^m, d).
        \]
    \end{enumerate}
    See Figure \ref{fig:schematic} for a schematic illustration, and Figure \ref{fig:experimental} for 
    numerical simulations.
\end{theorem}
\begin{remark}
    \label{rmk:dbetter}
    We quickly remark that the proof of the $O(d)$ upper bound indicated above is significantly more general than that in the statement. This proof, which is presented via Theorem \ref{thm:d}, \textit{does not require i.i.d. coordinates}, and instead it suffices for the data $x$ to satisfy $x = (\xi_1 b_1, \ldots, \xi_n b_n)$ where $(\xi_1, \ldots, \xi_n) \sim \nu$ for some strongly log-concave $\nu$ with density $\propto e^{-V(x)}$ with 
    \[
    K^{-1} \cdot I_n    \preceq \nabla^2 V(x) \preceq K \cdot I_n,
    \]
    and $(b_1, \ldots, b_n)$ is sampled independently from an arbitrary distribution on sparsity patterns $\in \{0, 1\}^n$.
\end{remark}

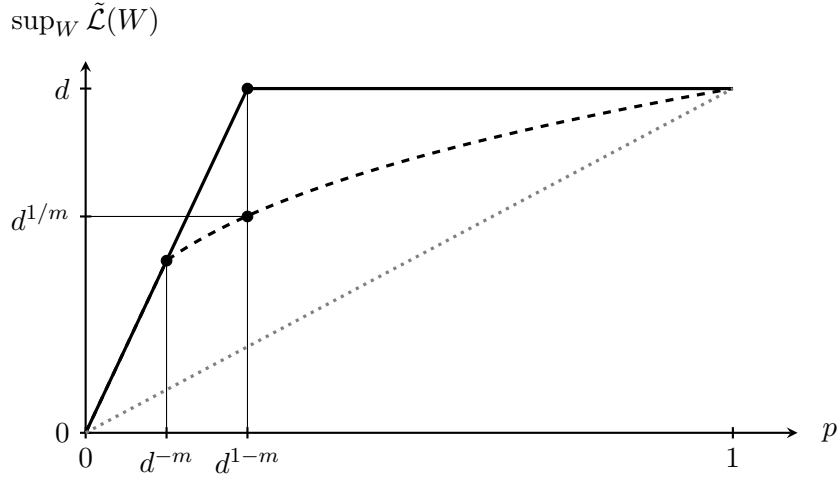
\begin{figure}[h]
    \centering
    \begin{tikzpicture}
        \pgfmathsetmacro{\dval}{2}          % fixed d for the schematic
        \pgfmathsetmacro{\doth}{\dval^0.33}
        \pgfmathsetmacro{\dm}{\dval^3}
        \pgfmathsetmacro{\pstar}{1/\dm}    % this equals d^{-m} in the schematic
        \pgfmathsetmacro{\poth}{\dval/\dm}

        \begin{axis}[
            width=11cm, height=6.5cm,
            axis lines=left,
            axis line style={black, thick},
            xmin=0, xmax=1.1,
            ymin=0, ymax=\dval*1.08,
            xtick={0,\pstar,\poth,1},
            xticklabels={$0$,$d^{-m}$,$d^{1 - m}$,$1$},
            ytick={0,\doth,\dval},
            yticklabels={$0$,$d^{1/m}$,$d$},
            tick style={black, thick},
            clip=false,
            xmajorgrids=false, ymajorgrids=false,
            ]

            % Upper bound: min(p d^m, d)
            \addplot[very thick, black, domain=0:1, samples=400]
            {min(x*\dm, \dval)};

            % Lower bound: min(p d^m, d p^{1/m})
            \addplot[very thick, black, dashed, domain=0:1, samples=400]
            {min(x*\dm, \dval*pow(x,1/3))};

            % Also draw pd
            \addplot[very thick, gray, dotted, domain=0:1, samples=400]
            {x*\dval};

            \draw[] (0, \doth) -- (\poth, \doth);
            \draw[] (\pstar, 0) -- (\pstar, \pstar*\dm);
            \draw[] (\poth, 0) -- (\poth, \dval);

            \filldraw[black] (\poth, \doth) circle(2pt);
            \filldraw[black] (\pstar, \pstar*\dm) circle(2pt);
            \filldraw[black] (\poth, \dval) circle(2pt);

            \node at (1.15, 0) { $p$};
            \node at (0, 2.4) {$\sup_W \L(W)$};

        \end{axis}
    \end{tikzpicture}
    \caption{Illustration of our results in the case when all the coordinates are i.i.d.
    and satisfy the strong log-concavity assumption. The solid line is the upper bound
    $\min(p d^m, d)$, and the dashed line is the lower bound $\min(p d^m, d p^{1/m})$.
    The gray dotted line is the performance
    of unsuperposed (orthogonal columns in $W$) solutions, i.e., $pd$. 
    Note that this diagram is purely schematic, since our results hide constants.}
    \label{fig:schematic}
\end{figure}
\begin{proof}
    The proof of this theorem occupies Sections \ref{sec:upper} and \ref{sec:lower} below.
\end{proof}

\begin{figure}
    \centering
    \input{full_loss.tex}
    \caption{Experimental loss curves for $d = 10, 15$ with $m = 3$ and $n = 6000$. We plot numerically optimized values of $\log_d(\L)$ v.s. $\log_d(p)$. Note that at small $p$ the data suggests a linear relationship, as proved by the tight upper and lower bounds above in this case. For larger $p$, the situation remains ambiguous. The ``slope-$1/3$'' lines are drawn such that they satisfy the observed data at $p = d^{-m} = d^{-3}.$}
    \label{fig:experimental}
\end{figure}
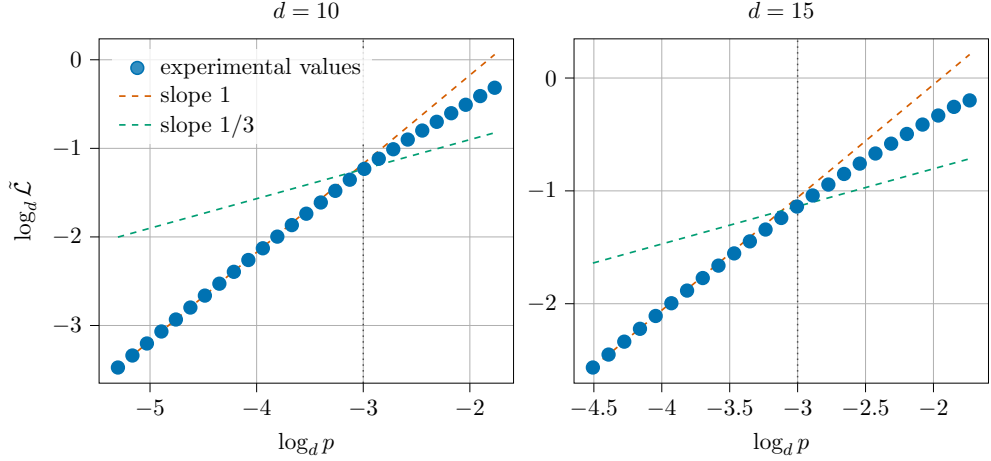

\begin{remark}
    A word on the notation. Throughout this paper, $m$ and $\mu$ will be fixed, and all asymptotic
    notations $\T{\cdot}, O(\cdot), \Om(\cdot)$ will hide constants depending only on $m$ and $\mu$.
    To simplify inequalities, $\les$ and $\gtr$ notations are also adopted. We remind the reader
    that $a \les b$ is equivalent to $a = O(b)$, and $a \gtr b$ is equivalent to $a = \Om(b)$.
\end{remark}

\section{Upper bounds}
\label{sec:upper}

\subsection{The $O(pd^m)$ bound}
This section will be devoted to one part of the upper bound in Theorem
\ref{thm:main}, restated here for convenience.
\begin{theorem}
    \label{thm:pdm}
    Under the usual assumptions
    \ref{thm:main}
    \[
        \sup_W \L(W) = O(pd^m),
    \]
    where $\L(W)$ is defined in \eqref{eq:loss}.
\end{theorem}
\begin{proof}
    Recall that
    \[
        \L(W) = 2 \cdot \E\inn{x, (W^T W x)^\om} - \E\norm{(W^T W x)^\om}_2^2,
    \]
    where $v^\om$ represents the pointwise $m$-th power of a vector $v$.
    A rather fruitful observation is that for any vector $u, v$,
    \[
        \inn{u, v}^m = \inn{u^{\otimes m}, v^{\otimes m}},
    \]
    where $u^{\otimes m}$ is the $m$-fold tensor product%
    \footnote{
        If $u \in \R^d$, then $u^{\otimes m} \in \R^{d^m}$ is defined by
        \[
            u^{\otimes m}_{i_1, \ldots, i_m} = u_{i_1} u_{i_2} \cdots u_{i_m},
        \]
        for all \textit{$m$-tuples} $(i_1, \ldots, i_m) \in \{1, \ldots,
        d\}^m$ (note
        that there are $d^m$ such tuples, which are identified with $\{1,
        \ldots, d^m\}$.).
    }
    of $u$ with itself. Thus, we can rewrite the above as
    \[
        \L(W) = 2 \cdot \E\inn{x, C_W y} - \E\norm{C_W y}_2^2,
    \]
    where $y = (W x)^{\otimes m} \in \R^{d^m}$ and $C_W \in \R^{n \times d^m}$ is the matrix
    defined via setting the rows to be
    \[
        (C_W)_i = (W^T_i)^{\otimes m}, \quad i = 1, \ldots, n.
    \]
    Now, fix $W$, and let us maximize over all choices of $C \in \R^{n \times
    d^m}$, yielding
    \def\M{\mathcal{M}}
    \begin{align*}
        \L(W) &\leq \M(W) \df \sup_{C} \Rnd{2 \cdot \E\inn{x, C y} -
        \E\norm{C y}_2^2}.
    \end{align*}
    One may explicitly compute a closed form for the right-hand side, given its
    linear nature, as achieved in Lemma \ref{lem:linearproblem} to obtain
    \[
        \M(W) = \tr\Rnd{M \Sig^\dagger M^T}, \quad \Sig = \E[y y^T], \quad M = \E[x y^T],
    \]
    where $\Sig^\dagger$ is the \textit{Moore-Penrose pseudoinverse} of $\Sig$. 
    Suppose $r \df \rk \Sig \leq d^m$, and diagonalize it as
    \[
        \Sig = \sum_{i=1}^r \lam_i u_i u_i^T,
    \]
    where $\lam_1, \ldots, \lam_r > 0$ are the nonzero eigenvalues and $u_1,
    \ldots, u_r$ are the corresponding orthonormal eigenvectors. Plugging this
    back into the expression for $\M(W)$ yields
    \begin{align}
         \label{eq:epzm}
         \M(W) &= \sum_{i=1}^r \f{1}{\lam_i} \cdot \norm{M u_i}_2^2 = \sum_{i =
         1}^r \norm{\E\Box{x \cdot \f{\inn{y, u_i}}{\sqrt{\lam_i}}}}_2^2.
    \end{align}
    Note that since $y = (W x)^{\otimes m}$, each entry of $y$, and therefore
    $\inn{y, u_i}$, are degree-$m$
    \textit{homogeneous} polynomials in the entries of $x$. Further, since $m$ is odd, 
    $\E \inn{y, u_i} = 0$ with variance
    \begin{align*}
         \E\Rnd{\frac{\inn{y, u_i}}{\sqrt{\lam_i}}}^2 &= \f{1}{\lam_i} \cdot
         \E[u_i^T y y^T u_i] \nonumber \\
                                                      &= \f{1}{\lam_i} \cdot
                                                      u_i^T \Sig u_i  = 1 \nonumber
    \end{align*}
    Therefore, in view of \eqref{eq:epzm}, since $r \leq d^m$, it will suffice to
    show that for any degree-$m$ homogeneous polynomial $\psi(x)$ with $\E \psi(x) = 0$
    and $\E \psi(x)^2 = 1$, we have
    \begin{align}
         \label{eq:t2}
        \norm{\E[x \psi(x)]}_2^2 = O(p).
    \end{align}
    Choose any $a \in \R^n$ with
    $\norm{a}_2 = 1$. Then,
    \begin{align}
        \Rnd{\E[\inn{a, x} \psi(x)]}^2 &\leq \E[\inn{a, x}^2] \cdot \E[\psi(x)^2] \nonumber \\
                                 &= p \mu_2,
    \end{align}
    due to 
    \[
        \E[\inn{a, x}^2] = \sum_{i=1}^n a_i^2 \cdot \E[x_i^2] = p \mu_2
    \]
    and $\E[\psi(x)^2] = 1$. Since $\sup_{\norm{a}_2 = 1} \E[\inn{a, x}
    \psi(x)] = \norm{\E[x \psi(x)]}_2$, this finishes the proof.
\end{proof}

\subsection{The global $O(d)$ bound}
In this section we prove the $O(d)$ upper bound under conditions more general than that in the statement of Theorem \ref{thm:main}, see Remark \ref{rmk:dbetter}.
\begin{theorem}
    \label{thm:d}
    Following Remark \ref{rmk:dbetter}, consider the following alternate data distribution: $x = (\xi_1 b_1, \ldots, \xi_n b_n)$ where $(\xi_1, \ldots, \xi_n) \sim \nu$ for some $(K^{-1}, K)$-strongly log-concave $\nu$, i.e., $\nu$ has density $\propto e^{-V(x)}$ with
    \begin{align}
        \label{eq:slc}
        K^{-1} \cdot I_n \preceq \nabla^2 V(x) \preceq K \cdot I_n,
    \end{align}
    and $(b_1, \ldots, b_n)$ is drawn independently from some arbitrary distribution on sparsity patterns $\in \{0, 1\}^n$.
    Then, defining $\L(W)$ analogously to \eqref{eq:loss}, we have
    \[
        \sup_W \L(W) = O(d).
    \]
    It is clear that this version implies the one claimed in Theorem \ref{thm:main} with $V(x) = \sum_{i = 1}^n v(x_i)$.
\end{theorem}

\begin{proof}
    Conditioning on the collection of surviving coordinates $S = \{1 \leq i \leq
    n: b_i = 1\}$, 
    and replacing
    $n$ by $|S|$, it may be
    quickly seen (see footnote\footnote{Two key facts are used here: (i) a marginal of a strongly
    log-concave measure is also strongly log-concave with the same constants and (ii) if $x$ is only supported on a subset of coordinates $S$ and $x_i = 0$ elsewhere, then $\inn{x, \phi(W^T W x)} - \norm{\phi(W^T W x)}^{2} \leq \inn{x_S, \phi(W_S^T W_S x_S)} - \norm{\phi(W_S^T W_S x_S)}^{2}$ where $x_S$ is $x$ restricted to the coordinates in $S$ and $W_S \in \R^{d \times |S|}$ is $W$ restricted to the columns with indices in $S$.
    })
        that it suffices to prove the theorem in the \textit{absence of sparsity} i.e., when
    $x \sim \nu$ directly. We assume this throughout the remainder of the
    proof.

    Fix $W \in \R^{d \times n}$, and construct%
    \footnote{
    This can be constructed, for instance, by
    using the SVD $W = U^T \Sigma V$ where $U \in \R^{d \times d}$ and $V \in
    \R^{n \times n}$ are orthogonal matrices, and $\Sigma \in \R^{d \times n}$ is
    diagonal with nonnegative entries, such that $\Sigma_{ij} = 0$ for all $j > d$. Then, set
    $P$ to be the first $d$ rows of $V$.
    }
    $P \in \R^{d  \times n}$ as an orthogonal projection so that $W = W'P$,
    where $W' \in \R^{d \times d}$. To control the effects of the nonlinearity,
    it will be prudent to condition on the projected information, i.e., $Px$.
    To that end, let us rewrite $\L(W)$ via the law of total expectation:
    \begin{align}
        \label{eq:sth}
        \L(W) &= 2 \cdot \E\inn{x, \phi(W^T W x)} - \E\norm{\phi(W^T W
        x)}_2^2 \nonumber \\
              &= \E\Box{\E\Box{2 \cdot \inn{x, \phi(W^T W x)} -
                      \norm{\phi(W^T W x)}_2^2 \Big| Px}} \nonumber \\
              &= \E\Box{2 \cdot \inn{\E[x | Px], \phi(W^T W x)} -
                      \norm{\phi(W^T W x)}_2^2},
    \end{align}
    since $Wx$, and thus $\phi(W^T W x)$, is measurable with respect to $Px$.
    For any vectors $u, v \in \R^n$, we have $2\inn{u, v} - \norm{v}_2^2 \leq
    \norm{u}_2^2$, and therefore,
    \[
        \L(W) \leq \E\norm{\E[x | Px]}_2^2,
    \]
    which yields
    \begin{align}
         \label{eq:dmain}
        \sup_W \L(W) \leq \sup_P \E\norm{\E[x | Px]}_2^2,
    \end{align}
    where the supremum is over all orthogonal projections $P$ of rank $d$. The
    remainder of the proof is devoted to bounding the right-hand side for any
    such $P$. Since a rotation of coordinates will preserve the strong log-concavity condition
    in \eqref{eq:slc}, we
    may assume without loss of generality that $P$ is the projection onto the
    first $d$ coordinates, and thus, all we need to do is to bound
    \[
        \E\norm{\E[x | x_1, \ldots, x_d]}_2^2 = \E\Box{\sum_{i=1}^d x_i^2} +
        \E\Box{\norm{\E[z | y]}^2},
    \]
    where $y \df (x_1, \ldots, x_d)$ and $z \df (x_{d+1}, \ldots, x_n)$, and
    $(y, z)$ is jointly strongly log-concave with the same constants as
    in \eqref{eq:slc}.
    The first term is already $O(d)$, so we focus on the second
    term.

    \renewcommand{\dj}[1]{\f{\partial{#1}}{{\partial y_j}}}

    Denote $m(y) \df \E[z | y]$. It suffices to show that $\E\norm{m(y)}^2 = O(d)$.
    Observe that since
    \[
        m(y) = Z(y)^{-1} \int z e^{-V(y, z)} \d z, \quad Z(y) \df \int
        e^{-V(y, z)} \d z,
    \]
    by Lemma \ref{lem:dergibbs} we have for any $i = 1, \ldots, {n - d}$ and
    $j = 1, \ldots, d$ the identity
    \[
        \dj{m(y)_i} = -\Cov\Rnd{z_i, \dj{V(y, z)}\; \Big|\; y}.
    \]
    Fix any two unit vectors $u \in \R^{n - d}$ and $v \in \R^d$, and consider
    $f(t) \df \inn{u, m(y + tv)}$. Then,
    \begin{align*}
        f'(0) &= \sum_{i = 1}^{n - d} u_i \frac{\d m(y +
        tv)_i}{\d t}\Big|_{t = 0} \\
              &= -\Cov\Rnd{\inn{u, z}, \inn{v, \nabla_y V(y, z)}\; \Big|\; y }.
    \end{align*}
    Therefore by Cauchy-Schwarz,
    \[
        |f'(0)|^2 \leq \Var\Rnd{\inn{u, z} | y} \cdot
        \Var\Rnd{\inn{v, \nabla_y V(y, z)} | y}.
    \]
    Even given $y$, $\inn{u, z}$ is a strongly log-concave random variable with
    constant $K^{-1}$, so by the \textit{Poincar\'e inequality}\footnote{
    For a differentiable function $f$ and a strongly log-concave measure $\propto e^{-V(x)}$
    with $\nabla^2 V(x) \succeq K^{-1} I$, the Poincar\'e inequality states that
    \[
        \Var(f(x)) \leq K \cdot \E\norm{\nabla f(x)}^2.
    \]
    See the discussion in \cite{chewi2023log} following Theorem 2.2.9.
    }
    for strongly
    log-concave measures, $\Var\Rnd{\inn{u, z} | y} \leq K$. On the other hand,
    invoking a stronger version of Poincar\'e inequality, known as the
    \textit{Brascamp-Lieb inequality} (\cite{brascamp1976extensions}, see also 
    \cite[Theorem 2.2.9]{chewi2023log} and the subsequent discussion there)
    on $g(z) = \inn{v, \nabla_y V(y, z)}$ yields
    %\red{check dimensions}
    \begin{align*}
        \Var\Rnd{\inn{v, \nabla_y V(y, z)} | y} &\leq \E\Box{\nabla_z g(z)^T \
        \Rnd{\nabla_z^2 V(y, z)}^{-1} \nabla_z g(z) \Big| y} \\
                                                &= \E\Box{\inn{v,
                                                        M_{zy}^TM_{zz}^{-1}M_{zy}
                                                    v} \Big| y} \\
                                                &\leq
                                                \norm{M_{zy}^TM_{zz}^{-1}M_{zy}}_{\op}
                                                \\
                                                &\overset{(a)}{\leq} \norm{M}_{\op} \leq K,
    \end{align*}
    where
    \begin{align}
         \label{eq:fasr}
        M = \nabla^2 V(y, z) = \begin{pmatrix}
            M_{yy} & M_{zy}^T \\
            M_{zy} & M_{zz}
        \end{pmatrix},
    \end{align}
    and step $(a)$ is via Lemma \ref{lem:blockinv}. All this allows us
    to conclude that $|f'(0)|^2 \leq K^2$, and since $u$ and $v$ were arbitrary unit
    vectors, we have that $m(y)$ is $K$-Lipschitz in $y$. Then, if $y'$ is an
    independent copy of $y$,
    \begin{align*}
        \E\norm{m(y) - m(y')}^2 &\leq K^2 \E\norm{y - y'}^2\\
                                &= 2 K^2 \E\norm{y}^2 \\
                                &= 2 K^3 d
    \end{align*}
    using the Poincar\'e inequality again to derive
    $\E[y_i^2] \leq K$. Finally, a
    simple computation shows that
    \[
        \E\norm{m(y) - m(y')}^2 = 2 \cdot \E\norm{m(y)}^2
    \]
    since $\E[m(y)] = \E[z] = 0$ by symmetry of $\nu$.
\end{proof}

\section{Lower bounds}
\label{sec:lower}

Our proof of the lower bound will depend upon the construction of a
particular type of matrix exhibiting  essentially optimal
closeness (as implied by the celebrated \textit{Welch bound}, see \cite{welch1974lower})
to a large identity matrix, despite being of much lower rank. This is described in the 
following lemma. See Remark \ref{rmk:heuristic} below for a heuristic explanation of why such
a matrix is useful.
\begin{lemma}
    \label{lem:const}
    For all sufficiently large $d$ (depending on $m$), there is a matrix $M \in \R^{d^m
    \times d^m}$ of the form $M = V^T V$ with $V \in \R^{d \times d^m}$ (so that
    $\rk M \leq d$)
    satisfying
    \begin{align}
         \label{eq:ltab}
         M_{ii} &= 1, \quad \forall i \nonumber, \\
         \Abs{M_{ij}} &\leq C_m d^{-1/2} < 1, \quad \forall i \neq j,
    \end{align}
    where $C_m$ depends only on $m$.
\end{lemma}
Direct randomized constructions will yield $\les \sqrt{\log d / d}$ in \eqref{eq:ltab}, therefore a
stronger argument is necessary to achieve Lemma \ref{lem:const}. A suitable construction is provided in the
Appendix as Lemma \ref{lem:consttwo}, invoking some strong tools from finite field theory.

Armed with this lemma, we now proceed to prove the lower bound in Theorem \ref{thm:main}, 
encapsulated in the
following theorem.

\begin{theorem}
    \label{thm:lower}
    Under the usual assumptions we have
    \[
        \sup_W \L(W) \gtr \min(pd^m, dp^{1/m}).
    \]
\end{theorem}
\begin{proof}
    \def\Mk{M^{(k)}}
    \def\Mdm{M^{(d^m)}}
    For some $k \leq d^m$ to be optimized over later, choose $A = W^T W$ to be a matrix of the form
    \[
        A = \begin{pmatrix}
            t\Mk & 0 \\
            0 & 0
        \end{pmatrix} \in \R^{n \times n},
    \]
    where $\Mk$ is the top-left $k \times k$ block of the matrix $M$ in Lemma \ref{lem:const} (so
    that $\Mdm = M$), and $t > 0$ is a parameter to be chosen soon. With this choice (and $W$ chosen
    so that $W^T W = A$), we have
    \begin{align}
         \label{eq:ohrv}
         \L(W) = 2t^m \cdot \E\inn{x, (\Mk x)^{\odot m}} - t^{2m} \cdot \E\norm{(\Mk x)}_{2m}^{2m}.
    \end{align}
    This can be optimized over $t$ to yield
    \begin{align}
        \sup_W \L(W) &\geq \sup_{k \leq d^m} \f{\Rnd{\E\inn{x, (\Mk x)^{\odot m}}}^2}
                                             {\Rnd{\E\norm{(\Mk x)}_{2m}^{2m}}}. \nonumber \\
    \end{align}
    The numerator is
    \begin{align}
         \E\inn{x, (\Mk x)^{\odot m}} &= \sum_{i=1}^k \E\Box{x_i \cdot \Rnd{\sum_{j = 1}^k M_{ij} x_j}^m} \nonumber \\
                                      &= \sum_{i=1}^k \sum_{\al : |\al| = m} \binom{m}{\al}
                                     \prod_{j = 1}^k M_{ij}^{\al_j} \cdot 
                                      \E\Box{x_i^{\al_i + 1} \prod_{j \neq i} x_j^{\al_j}} \nonumber
    \end{align}
    where we use the multinomial expansion with the terms indexed by multi-indices $\al \in \N^k$
    ($\N = \{0, 1, \ldots \}$) with $|\al| \df \sum_{j=1}^k \al_j = m$, and $\binom{m}{\al} \df
    \f{m!}{\prod_j \al_j!}$. Due to the symmetry of $\mu$,
    each $\al_j$ must be even for $j \neq i$ for the expectation to be nonzero, and $\al_i$ must be
    odd. This implies that for every $\al$, the corresponding term is either zero or positive. We
    can therefore lower bound this quantity by only considering the terms with $\al_i = m$:
    \begin{align}
         \label{eq:dhzi}
        \E\inn{x, (\Mk x)^{\odot m}} &\geq \sum_{i=1}^k M_{ii}^{m} \E[x_i^{m + 1}] = p\mu_{m +
        1} \cdot k \gtr pk.
    \end{align}
    We now turn to the denominator, which may be upper bounded as follows:
    \begin{align}
         \label{eq:vucw}
         \sum_{i = 1}^k \E\Box{\Rnd{\sum_{j = 1}^k M_{ij} x_j}^{2m}} 
         &\overset{(a)}{\les}
         \sum_{i = 1}^k  \max\Rnd{\sum_j \E\Abs{M_{ij} x_j}^{2m}, \Rnd{\sum_j \E\Abs{M_{ij} x_j}^2}^m} \nonumber \\
         &\les
         \max\Rnd{p\cdot\sum_{i = 1}^k\sum_{j = 1}^k |M_{ij}|^{2m},\; p^m \cdot \sum_{i =
         1}^k\Rnd{\sum_{j = 1}^k M^2_{ij}}^m} \nonumber \\
         % &\les
         % \max\Rnd{pk + pk^2d^{-m},\; p^m k \cdot \Rnd{1 + C_m k d^{-1}}^m} \nonumber \\
         &\les \max\Rnd{pk,\; p^m k \Rnd{1 + kd^{-1}}^m},
    \end{align}
    where step $(a)$ is due to \textit{Rosenthal's inequality}, see \cite{rosenthal1970subspaces}.
    Combining \eqref{eq:dhzi} and \eqref{eq:vucw} yields
    \begin{align}
         % \label{eq:hisb}
         \sup_W \L(W) &\gtr \sup_{k \leq d^m} \min\Rnd{pk,\; p^{2 - m} k \cdot (1 + k/d)^{-m}}
         \nonumber \\
                      &\gtr \sup_{d \leq k \leq d^m} \min\Rnd{pk,\; p^{2 - m} k^{1 - m} d^{m}} \nonumber
    \end{align}
    where restricting to $k \geq d$ allows the bound $(1 + k/d)^{-m} \gtr (k/d)^{-m}$. If $p \leq
    d^{-m}$, the choice $k = d^m$ yields $p^{2 - m}k^{1 - m} d^m = p \cdot p^{1 - m} d^{2m - m^2} \geq pd^m$ and thus
    $\sup_W \L(W) \gtr pd^m$. Otherwise, choose\footnote{The
        integer effect can be ignored by noting that since this choice is $\geq d$, the nearest
        integer is within a factor of 2, which can be subsumed into the constant.} 
        $k = \f{d}{p} \cdot p^{1/m}$ (which is a valid choice since $p > d^{-m}$ implies $\f{d}{p}
        p^{1/m} < d^m$), so that $pk = dp^{1/m} = p^{2 - m} k^{1 - m} d^m$. This completes the
        proof.
\end{proof}

\begin{remark}
\label{rmk:heuristic}
\def\tB{\tilde{B}}
\def\tlam{\tilde{\lambda}}
\def\tV{\tilde{V}}
    We quickly note a heuristic that was in fact the guiding principle for our construction above.
Let us begin by extracting the first-order in $p$ from \eqref{eq:loss}. To that
end, we expand
\begin{align*}
     \L(W) &= 2 \cdot \E\inn{x, \phi(W^T W x)} - \E\norm{\phi(W^T W x)}_2^2 \nonumber \\
           &= p \cdot \Rnd{2\mu_{m + 1} \cdot \sum_{i = 1}^n A_{ii}^m  - \mu_{2m} \sum_{i, j=1}^n
           A_{ij}^{2m}} + \text{terms of order $p^2$},
\end{align*}
where $A = W^T W$. The quantity inside the parenthesis may be simplified by setting $A = cB$ where
$c = \Rnd{\mu_{m + 1}/\mu_{2m}}^{1/m}$, yielding
\begin{align*}
     \L(W) &= p \cdot \f{\mu_{m + 1}^2}{\mu_{2m}} \cdot \cF(B) + \text{terms of
     order $p^2$},
\end{align*}
where the $\mu$-free objective $\cF(B)$ is defined as
\begin{align}
    \label{eq:fb}
    \cF(B) = 2 \cdot \sum_{i=1}^n B_{ii}^m - \sum_{i, j=1}^n B_{ij}^{2m}.
\end{align}
which must be maximized over positive semidefinite matrices $B \in \R^{n \times n}$ with $\rk B \leq d$. Since, at least for small enough $p$, the first order should be dominant, this is an important objective to optimize, and the answer to this problem should dictate the very sparse limit.
Define $\tB = B^{\odot m}$ (pointwise power), so that
\[
    \cF(B) = \tr(2\tB - \tB^2) \leq \sum_{i=1}^{r} 2\tlam_i - \tlam_i^2 \leq r,
\]
where $r = \rk \tB$.  If $B = V^T V$ for $V \in \R^{d \times n}$, then $\tB = \tV^T \tV$ where
the columns of $\tV$ are the $m$-fold tensor products of the columns of $V$, and thus $\tV \in
\R^{d^m \times n}$. Consequently, $r \leq
d^m$, and thus $\cF(B)$ can be at most $d^m$. One may easily check that the construction in 
Lemma \ref{lem:const} (with an optimal scaling) achieves this upper bound up to constant factors.
\end{remark}

\section{Open problems and future directions}
During the development of the work in this article, we encountered a selection of interesting
problems, some of which we list below.

\begin{enumerate}
    \item Perhaps the most interesting question is the true correct order of $\sup_W \L(W)$,
        under natural conditions on $\mu$. Note that our $O(d)$ upper bound, which is really the
        only piece that we can expect to be loose, does not rely on the particular form of $\phi$
        considered here.
    \item We expect Theorem \ref{thm:d}, which is a direct 
        consequence of an $O(d)$ upper bound on \eqref{eq:dmain}, to hold under
        a wider range of conditions on $\mu$ than strong log-concavity.
    \item We are also interested in a tight understanding of the first-order loss \eqref{eq:fb},
        which dictates the loss behavior under very sparse inputs. This is particularly appealing
        to us due to its self-contained nature as an optimization problem over low-rank positive
        semidefinite matrices. It does not appear to yield easily to a direct spectral analysis since
        pointwise powers, in general, behave rather poorly with respect to the spectrum.
    \item The reader will note that the construction in Lemma \ref{lem:const}, used in the lower
        bound, is somewhat involved. Does the gradient descent on the loss truly converge to such
        highly structured solutions? This does not seem likely, and therefore it remains interesting to
        understand the true nature of solutions obtained via training.
\end{enumerate}

\bibliographystyle{plainnat}
\bibliography{ref}

% see https://tex.stackexchange.com/questions/174887/link-to-appendix-from-anywhere-in-the-document-goes-to-the-wrong-place
\appendix

% \crefalias{section}{appendix} % uncomment if you are using cleveref

\section{Lemmas required in the proofs}

\begin{lemma}
    \label{lem:dergibbs}
    Suppose $\{\mu_s\}_s$ is a family of Gibbs measures parameterized by $s \in \R$ where
    \[
        \mu_s(\d x) = Z(s)^{-1} e^{-V(x, s)} \d x, \quad Z(s) = \int
        e^{-V(x, s)} \d x,
    \]
    with $V$ smooth\footnote{Assumed to have sufficient decay at infinity to ensure
    all quantities involved are well-defined.}.
    Then, for any $f : \R^n \to \R$ (smooth and integrable)
    \[
        \frac{\d}{\d s} \E_{\mu_s} f(x) = -\Cov_{\mu_s}\Rnd{f(x), \frac{\partial
        V}{\partial s}(x, s)}.
    \]
\end{lemma}
\begin{proof}
    \def\ds{\partial_s}
   Denote $p(x, s) = Z(s)^{-1} e^{-V(x, s)}$. Then,
    \begin{align*}
        \frac{\d}{\d s} \E_{\mu_s} f(x) &= \frac{\d}{\d s} \int f(x) p(x, s)
        \d x \\
        &= \int f(x) \ds p(x, s) \d x \\
        &= \int f(x) p(x, s) \cdot \ds \log p(x, s) \d x \\
        &= \E\Box{ f(x) \cdot \Rnd{-\ds V(x, s) - \frac{\d}{\d s} \log Z(s)} } \\
        &= -\E\Box{ f(x) \cdot \ds V(x, s) } - \E[f(x)] \cdot
        \frac{\d}{\d s} \log Z(s) \\
        &= -\E\Box{ f(x) \cdot \ds V(x, s) } + \E[f(x)] \cdot \E\Box{\ds V(x,
        s)} \\
        &= -\Cov\Rnd{f(x), \ds V(x, s)},
    \end{align*}
    since
    \[
        \frac{\d}{\d s} \log Z(s) = Z(s)^{-1} \frac{\d}{\d s} Z(s) = -\E\Box{\ds
        V(x, s)}.
    \]
\end{proof}

\begin{lemma}
    \label{lem:blockinv}
    Suppose $A$ is a positive definite block matrix of the form
    \[
        A = \begin{pmatrix}
            A_{11} & A_{12} \\
            A_{21} & A_{22}
        \end{pmatrix}, \quad A_{12} = A_{21}^T,
    \]
    where $A_{11}$ and $A_{22}$ are square matrices. Then,
    \[
        \norm{A_{12} A_{22}^{-1} A_{21}}_{\op} \leq \norm{A}_{\op}.
    \]
\end{lemma}
\begin{proof}
    Since $A \succ 0$, we have $A_{22} \succ 0$ as well. Thus, the Schur
    complement $S = A_{11} - A_{12} A_{22}^{-1} A_{21} \succ 0$ (a quick way to
    see  this is that this is the top-left block in the block-inverse of $A$). 
    Therefore,
    \[
        A_{12} A_{22}^{-1} A_{21} \preceq A_{11},
    \]
    which implies the desired operator norm bound since $\norm{A_{11}}_{\op} \leq
    \norm{A}_{\op}$.
\end{proof}
    
\begin{lemma}
     \label{lem:linearproblem}
    Suppose $x \in \R^q$ and $y \in \R^r$ are random vectors with all finite moments. Then,
    \[
        \sup_{C \in \R^{q \times r}} \Rnd{2 \cdot \E\inn{x, C y} - \E\norm{C
        y}_2^2} = \tr\Rnd{M \Sig^\dagger M^T},
    \]
    where $\Sig = \E[y y^T]$, $M = \E[x y^T]$, and $\Sig^\dagger$ is the
    Moore-Penrose pseudoinverse of $\Sig$.
\end{lemma}
\begin{proof}
    To begin, observe that the supremum is finite, since
    \[
        2\inn{x, C y} - \norm{C y}_2^2 \leq \norm{x}_2^2,
    \]
    and therefore the objective is bounded above by $\E\norm{x}_2^2$, finite by
    assumption.
    Let us begin by writing the objective as a deterministic function
    of $C$:
    \begin{align*}
         \label{eq:dgsh}
        \E[2\inn{x, C y} - \norm{C y}_2^2] &= 
                2 \cdot \tr\Rnd{C \cdot \E[y x^T]} - \tr\Rnd{C \cdot \E[y y^T] \cdot C^T} \\
                                         &= 2 \cdot \tr(C M^T) - \tr(C \Sig C^T), 
    \end{align*}
    where $\Sig = \E[y y^T]$ and $M = \E[x y^T]$. Note that $\Sig \succeq 0$,
    and therefore, the objective is a concave function of $C$.

    Consider the first order optimality condition. Differentiating with respect to
    $C$ and setting to zero, we obtain
    \[
        \nabla_C \Rnd{2 \cdot \tr(C M^T) - \tr(C \Sig C^T)} = 2M - 2 C \Sig = 0,
    \]  
    a solution to which would exist if and only if each row $M_i$ of $M$ lies 
    in the row space of $\Sig$. In that case, note that the choice of $C = M\Sig^\dagger$
    is valid, since by hypothesis, there is some $D$ satisfying $M = D \Sig$, and therefore,
    \[
        C \Sig = M \Sig^\dagger \Sig = D \Sig \Sig^\dagger \Sig = D \Sig = M,
    \]
    by properties of the pseudoinverse. Plugging this choice into the objective yields the value
    \[
        2 \cdot \tr(M\Sig^\dagger M^T) - \tr(M \Sig^\dagger \Sig \Sig^\dagger M^T) = \tr(M \Sig^\dagger M^T),
    \]
    as desired.

    If, however, some row $M_i$ does not lie in the row space of $\Sig$, construct $C$
    such that every row of $C$ is zero except the $i$-th row $C_i$. Then, the objective reduces to
    \[
        2 C_i M_i^T - C_i \Sig C_i^T \in \R.
    \]
    By choosing $C_i$ to be a multiple of a vector orthogonal to the row space of $\Sig$ but with
    positive inner product with $M_i$, the objective may be made arbitrarily large, contradicting
    the finiteness of the supremum. Consequently this case must be impossible, concluding the proof.
\end{proof}

We also complete the proof of Lemma \ref{lem:const}, restated here for convenience.
\begin{lemma}
    \label{lem:consttwo}
    Fix $m$. For all sufficiently large $d$, there is a matrix $M \in \R^{d^m
    \times d^m}$ of the form $M = V^T V$ with $V \in \R^{d \times d^m}$ (so that
    $\rk M \leq d$) satisfying
    \begin{align}
         \label{eq:ltabrestated}
         M_{ii} &= 1, \quad \forall i, \\
         \Abs{M_{ij}} &\leq C_m d^{-1/2} < 1, \quad \forall i \neq j,
    \end{align}
    where $C_m$ depends only on $m$.
\end{lemma}
\begin{proof}
    \def\Fr{\mathbb{F}_{2^r}}
    \def\F{\mathbb{F}_2}
    The proof of this will invoke tools from classical field theory.
    Constructions of this type have been used extensively in the literature
    on deterministic sensing matrices for compressed sensing and low-correlation
    sequences, see \cite{wang2013new, yu2011additive, schmidt2011sequence}, and
    references therein. For the sake of convenience and completeness, we exhibit
    the main ideas specialized to our simple requirements, relying only upon
    basic facts from field theory, except for the only deep fact used here,
    the Weil bound on character sums; see \cite{lidl1983finite}, also
    \cite{koppartyelementary} for a simpler exposition of the result.

    We use $\mathbb{F}_n$ for the unique field of order $n$, if it exists.
    First assume $d = 2^r$, a power of two, and consider the field
    $\Fr$ (constructed from $\F[x]$ modulo an irreducible polynomial of
    degree $r$). Let $\tr : \Fr \to \F$ be the trace\footnote{
        It may be shown that this is the usual matrix trace of the $\F$-linear map
        $L_a : \Fr \to \Fr$ defined as $L_a(z) = a z$ for fixed $a \in \Fr$.
    } map defined as
    \[
        \tr(a) = a + a^2 + a^{2^2} + \cdots + a^{2^{r - 1}}.
    \]
    Note that for any $a \in \Fr$,
    $a^2 = a$ if and only if $a \in \F$ (since this is degree 2 polynomial and
    all the elements in $\F$ are roots). Therefore, since $(x + y)^2 = x^2 +
    y^2$ in $\F$, we have $\tr(a) \in \F$ for all $a \in \Fr$ (since $\tr(a)^2 =
    \tr(a)$). Further, note that $\tr$ is $\F$-linear, that is, for any $a, b \in
    \Fr$, $\tr(a + b) = \tr(a) + \tr(b)$ and $\tr(ca) = c \tr(a)$ for all $c \in \F$.
    
    Now consider the character $\chi : \F \to \R$ defined as 
    \[
        \chi(0) = 1, \quad \chi(1) = -1.
    \]
    (usually characters are defined to take values in the complex unit
    circle, but this is a real-valued character), and define the function
    \[
        \psi(a) = \chi(\tr(a)) = (-1)^{\tr(a)}, \quad a \in \Fr.
    \]
    Note that $\psi$ is an additive character of $\Fr$, that is, for any $a, b \in
    \Fr$, $\psi(a + b) = \psi(a) \cdot \psi(b)$.
    \def\a{\mathbf{a}}
    \def\b{\mathbf{b}}

    Define the matrix $V \in \R^{d \times d^m}$ with rows indexed by $\Fr$ and 
    columns indexed by tuples $\a = (a_0, \ldots, a_{m - 1}) \in \Fr^m$ (i.e., $d^m$ columns)
    as follows:
    \[
        V_{x; \a} = \f{1}{\sqrt{d}} \cdot \psi\Rnd{
            \sum_{j = 0}^{m - 1} a_j x^{2j + 1}
        }.
    \]
    We will now check that $V$ satisfies our requirements. Note that for any $x$
    and two tuples $\a = (a_0, \ldots, a_{m - 1})$, $\b = (b_0, \ldots, b_{m - 1})$,
    \begin{align*}
        V_{x; \a} \cdot V_{x; \b} &=
        \f{1}{d} \cdot \psi\Rnd{
            \sum_{j = 0}^{m - 1} (a_j - b_j) x^{2j + 1}
        }.
    \end{align*}
    (the minus sign is due to the observation that $\psi(-z) = \psi(z)$). If $\a
    = \b$, then $V_{x; \a} \cdot V_{x; \b} = 1/d$, and thus $(V^T V)_{\a;\a} = 1$
    for all $\a$, satisfying the diagonal requirement. If $\a \neq \b$, then
    \begin{align}
         \label{eq:wema}
         (V^T V)_{\a;\b} &= \f{1}{d} \cdot \sum_{x \in \Fr} \psi(P(x)), \quad
         \text{where},\\
         P(x) &= \sum_{j = 0}^{m - 1} (a_j - b_j) x^{2j + 1} \not\equiv 0. \nonumber
    \end{align}
    At this point we will invoke Weil's character sum bounds (specialized to
    our case) which states
    that for any nontrivial additive character $\psi$ of $\Fr$ and polynomial
    $P$ of degree $n$ over $\Fr$ which is not of the form $Q(x)^2 - Q(x) + c$ for
    some polynomial $Q$ and constant $c$, we have
    \[
        \Abs{\sum_{x \in \Fr} \psi(P(x))} \leq (n - 1) \sqrt{2^r} = (n - 1)
        \sqrt{d}.
    \]
    Note that any polynomial $P$ of odd degree cannot be of the form
    $Q(x)^2 - Q(x) + c$ since the latter is always of even degree. Note that as
    long as the $P$ in \eqref{eq:wema} is nonzero, it is of odd degree, since
    each term is of odd degree. Therefore, applying Weil's bound with $n =
    \deg P \leq 2m - 1$, we have
    \[
        \Abs{(V^T V)_{\a;\b}} \leq \f{1}{d} \cdot (2m - 2) \sqrt{d} = \f{2m -
        2}{\sqrt{d}},
    \]
    completing the proof for $d$ a power of two with $C_m = 2m - 2$.

    For a general $d$, choose $r$ such that $2^r \leq d < 2^{r + 1}$, and
    repeat the same construction as above, but with $2m$ instead of $m$ (so that
    the number of columns is $2^{2rm} \geq d^m$ as long as $d$ is large enough).
    Then, select a $2^r \times d^m$ submatrix of $V$ (say, the first $d^m$
    columns),
    and extend it to a $d \times d^m$ matrix by adding zero rows at the bottom.
    This completes the construction.
\end{proof}

\end{document}

%% file: full_loss.tex
% This file was created with tikzplotlib v0.10.1.
\begin{tikzpicture}[scale=0.8]

\definecolor{chocolate213940}{RGB}{213,94,0}
\definecolor{darkcyan0114178}{RGB}{0,114,178}
\definecolor{darkcyan0158115}{RGB}{0,158,115}
\definecolor{darkgray176}{RGB}{176,176,176}

\begin{groupplot}[group style={group size=2 by 1}]
\nextgroupplot[
legend cell align={left},
legend style={
  fill opacity=0.8,
  draw opacity=1,
  text opacity=1,
  at={(0.03,0.97)},
  anchor=north west,
  draw=none
},
tick align=outside,
tick pos=left,
title={\(\displaystyle d=10\)},
x grid style={darkgray176},
xlabel={\(\displaystyle \log_{d}p\)},
xmajorgrids,
xmin=-5.47773116948125, xmax=-1.59030906276756,
xtick style={color=black},
y grid style={darkgray176},
ylabel={\(\displaystyle \log_{d}\L\)},
ymajorgrids,
ymin=-3.65114519157863, ymax=0.236276915135057,
ytick style={color=black}
]
\addplot [line width=0.72pt, darkcyan0114178, mark=*, mark size=3, mark options={solid}, only marks]
table {%
-5.30103016463063 -3.47444418672801
-5.16510631478244 -3.33877105486949
-5.02918246306018 -3.20318175037114
-4.89325861711414 -3.067704881341
-4.75733474313267 -2.93239320207349
-4.62141089239396 -2.79725508824061
-4.48548706363159 -2.6624604972405
-4.34956321699251 -2.5280124778874
-4.21363935092383 -2.3940590646652
-4.07771550714421 -2.26062946606189
-3.94179166814405 -2.12823570020432
-3.80586780597046 -1.99670600535054
-3.66994394624151 -1.86664744890975
-3.53402010287867 -1.73785604736848
-3.39809624200602 -1.61118356129391
-3.26217239421951 -1.47967634968477
-3.12624855413446 -1.35393571054838
-2.99032472350338 -1.23177202518336
-2.85440085505976 -1.1172406115129
-2.71847699592281 -1.01007841270098
-2.58255315094671 -0.900670499062681
-2.44662929833285 -0.798982269807326
-2.31070545212212 -0.700369477778393
-2.17478159536402 -0.603442221352721
-2.03885773326491 -0.507980879565638
-1.90293391617813 -0.409801394086856
-1.76701006761818 -0.315742741123516
};
\addlegendentry{experimental values}
\addplot [thick, chocolate213940, dashed]
table {%
-5.30103016463063 -3.47444418672801
-5.28327126967076 -3.45668529176815
-5.2655123747109 -3.43892639680829
-5.24775347975104 -3.42116750184843
-5.22999458479118 -3.40340860688856
-5.21223568983132 -3.3856497119287
-5.19447679487146 -3.36789081696884
-5.1767178999116 -3.35013192200898
-5.15895900495173 -3.33237302704912
-5.14120010999187 -3.31461413208926
-5.12344121503201 -3.2968552371294
-5.10568232007215 -3.27909634216953
-5.08792342511229 -3.26133744720967
-5.07016453015243 -3.24357855224981
-5.05240563519256 -3.22581965728995
-5.0346467402327 -3.20806076233009
-5.01688784527284 -3.19030186737023
-4.99912895031298 -3.17254297241036
-4.98137005535312 -3.1547840774505
-4.96361116039326 -3.13702518249064
-4.9458522654334 -3.11926628753078
-4.92809337047353 -3.10150739257092
-4.91033447551367 -3.08374849761106
-4.89257558055381 -3.0659896026512
-4.87481668559395 -3.04823070769133
-4.85705779063409 -3.03047181273147
-4.83929889567423 -3.01271291777161
-4.82154000071437 -2.99495402281175
-4.8037811057545 -2.97719512785189
-4.78602221079464 -2.95943623289203
-4.76826331583478 -2.94167733793216
-4.75050442087492 -2.9239184429723
-4.73274552591506 -2.90615954801244
-4.7149866309552 -2.88840065305258
-4.69722773599533 -2.87064175809272
-4.67946884103547 -2.85288286313286
-4.66170994607561 -2.83512396817299
-4.64395105111575 -2.81736507321313
-4.62619215615589 -2.79960617825327
-4.60843326119603 -2.78184728329341
-4.59067436623616 -2.76408838833355
-4.5729154712763 -2.74632949337369
-4.55515657631644 -2.72857059841383
-4.53739768135658 -2.71081170345396
-4.51963878639672 -2.6930528084941
-4.50187989143686 -2.67529391353424
-4.484120996477 -2.65753501857438
-4.46636210151713 -2.63977612361452
-4.44860320655727 -2.62201722865466
-4.43084431159741 -2.6042583336948
-4.41308541663755 -2.58649943873493
-4.39532652167769 -2.56874054377507
-4.37756762671783 -2.55098164881521
-4.35980873175797 -2.53322275385535
-4.3420498367981 -2.51546385889549
-4.32429094183824 -2.49770496393563
-4.30653204687838 -2.47994606897576
-4.28877315191852 -2.4621871740159
-4.27101425695866 -2.44442827905604
-4.2532553619988 -2.42666938409618
-4.23549646703893 -2.40891048913632
-4.21773757207907 -2.39115159417646
-4.19997867711921 -2.3733926992166
-4.18221978215935 -2.35563380425673
-4.16446088719949 -2.33787490929687
-4.14670199223963 -2.32011601433701
-4.12894309727977 -2.30235711937715
-4.1111842023199 -2.28459822441729
-4.09342530736004 -2.26683932945743
-4.07566641240018 -2.24908043449756
-4.05790751744032 -2.2313215395377
-4.04014862248046 -2.21356264457784
-4.0223897275206 -2.19580374961798
-4.00463083256073 -2.17804485465812
-3.98687193760087 -2.16028595969826
-3.96911304264101 -2.1425270647384
-3.95135414768115 -2.12476816977853
-3.93359525272129 -2.10700927481867
-3.91583635776143 -2.08925037985881
-3.89807746280156 -2.07149148489895
-3.8803185678417 -2.05373258993909
-3.86255967288184 -2.03597369497923
-3.84480077792198 -2.01821480001936
-3.82704188296212 -2.0004559050595
-3.80928298800226 -1.98269701009964
-3.7915240930424 -1.96493811513978
-3.77376519808253 -1.94717922017992
-3.75600630312267 -1.92942032522006
-3.73824740816281 -1.9116614302602
-3.72048851320295 -1.89390253530033
-3.70272961824309 -1.87614364034047
-3.68497072328323 -1.85838474538061
-3.66721182832337 -1.84062585042075
-3.6494529333635 -1.82286695546089
-3.63169403840364 -1.80510806050103
-3.61393514344378 -1.78734916554116
-3.59617624848392 -1.7695902705813
-3.57841735352406 -1.75183137562144
-3.5606584585642 -1.73407248066158
-3.54289956360433 -1.71631358570172
-3.52514066864447 -1.69855469074186
-3.50738177368461 -1.680795795782
-3.48962287872475 -1.66303690082213
-3.47186398376489 -1.64527800586227
-3.45410508880503 -1.62751911090241
-3.43634619384517 -1.60976021594255
-3.4185872988853 -1.59200132098269
-3.40082840392544 -1.57424242602283
-3.38306950896558 -1.55648353106296
-3.36531061400572 -1.5387246361031
-3.34755171904586 -1.52096574114324
-3.329792824086 -1.50320684618338
-3.31203392912613 -1.48544795122352
-3.29427503416627 -1.46768905626366
-3.27651613920641 -1.4499301613038
-3.25875724424655 -1.43217126634393
-3.24099834928669 -1.41441237138407
-3.22323945432683 -1.39665347642421
-3.20548055936697 -1.37889458146435
-3.1877216644071 -1.36113568650449
-3.16996276944724 -1.34337679154463
-3.15220387448738 -1.32561789658476
-3.13444497952752 -1.3078590016249
-3.11668608456766 -1.29010010666504
-3.0989271896078 -1.27234121170518
-3.08116829464793 -1.25458231674532
-3.06340939968807 -1.23682342178546
-3.04565050472821 -1.2190645268256
-3.02789160976835 -1.20130563186573
-3.01013271480849 -1.18354673690587
-2.99237381984863 -1.16578784194601
-2.97461492488877 -1.14802894698615
-2.9568560299289 -1.13027005202629
-2.93909713496904 -1.11251115706643
-2.92133824000918 -1.09475226210656
-2.90357934504932 -1.0769933671467
-2.88582045008946 -1.05923447218684
-2.8680615551296 -1.04147557722698
-2.85030266016973 -1.02371668226712
-2.83254376520987 -1.00595778730726
-2.81478487025001 -0.988198892347396
-2.79702597529015 -0.970439997387534
-2.77926708033029 -0.952681102427673
-2.76150818537043 -0.934922207467811
-2.74374929041057 -0.917163312507949
-2.7259903954507 -0.899404417548088
-2.70823150049084 -0.881645522588226
-2.69047260553098 -0.863886627628365
-2.67271371057112 -0.846127732668503
-2.65495481561126 -0.828368837708642
-2.6371959206514 -0.81060994274878
-2.61943702569153 -0.792851047788919
-2.60167813073167 -0.775092152829057
-2.58391923577181 -0.757333257869196
-2.56616034081195 -0.739574362909334
-2.54840144585209 -0.721815467949472
-2.53064255089223 -0.704056572989611
-2.51288365593237 -0.686297678029749
-2.4951247609725 -0.668538783069888
-2.47736586601264 -0.650779888110026
-2.45960697105278 -0.633020993150165
-2.44184807609292 -0.615262098190303
-2.42408918113306 -0.597503203230442
-2.4063302861732 -0.57974430827058
-2.38857139121333 -0.561985413310719
-2.37081249625347 -0.544226518350857
-2.35305360129361 -0.526467623390996
-2.33529470633375 -0.508708728431134
-2.31753581137389 -0.490949833471273
-2.29977691641403 -0.473190938511411
-2.28201802145417 -0.45543204355155
-2.2642591264943 -0.437673148591688
-2.24650023153444 -0.419914253631827
-2.22874133657458 -0.402155358671965
-2.21098244161472 -0.384396463712104
-2.19322354665486 -0.366637568752242
-2.175464651695 -0.348878673792381
-2.15770575673513 -0.331119778832519
-2.13994686177527 -0.313360883872657
-2.12218796681541 -0.295601988912796
-2.10442907185555 -0.277843093952934
-2.08667017689569 -0.260084198993073
-2.06891128193583 -0.242325304033211
-2.05115238697597 -0.22456640907335
-2.0333934920161 -0.206807514113488
-2.01563459705624 -0.189048619153627
-1.99787570209638 -0.171289724193765
-1.98011680713652 -0.153530829233903
-1.96235791217666 -0.135771934274042
-1.9445990172168 -0.11801303931418
-1.92684012225693 -0.100254144354319
-1.90908122729707 -0.0824952493944573
-1.89132233233721 -0.0647363544345958
-1.87356343737735 -0.0469774594747343
-1.85580454241749 -0.0292185645148728
-1.83804564745763 -0.0114596695550113
-1.82028675249777 0.00629922540485017
-1.8025278575379 0.0240581203647117
-1.78476896257804 0.0418170153245736
-1.76701006761818 0.0595759102844347
};
\addlegendentry{slope $1$}
\addplot [thick, darkcyan0158115, dashed]
table {%
-5.30103016463063 -2.00200717222578
-5.28327126967076 -1.99608754057249
-5.2655123747109 -1.9901679089192
-5.24775347975104 -1.98424827726592
-5.22999458479118 -1.97832864561263
-5.21223568983132 -1.97240901395934
-5.19447679487146 -1.96648938230606
-5.1767178999116 -1.96056975065277
-5.15895900495173 -1.95465011899948
-5.14120010999187 -1.94873048734619
-5.12344121503201 -1.94281085569291
-5.10568232007215 -1.93689122403962
-5.08792342511229 -1.93097159238633
-5.07016453015243 -1.92505196073305
-5.05240563519256 -1.91913232907976
-5.0346467402327 -1.91321269742647
-5.01688784527284 -1.90729306577318
-4.99912895031298 -1.9013734341199
-4.98137005535312 -1.89545380246661
-4.96361116039326 -1.88953417081332
-4.9458522654334 -1.88361453916004
-4.92809337047353 -1.87769490750675
-4.91033447551367 -1.87177527585346
-4.89257558055381 -1.86585564420017
-4.87481668559395 -1.85993601254689
-4.85705779063409 -1.8540163808936
-4.83929889567423 -1.84809674924031
-4.82154000071437 -1.84217711758703
-4.8037811057545 -1.83625748593374
-4.78602221079464 -1.83033785428045
-4.76826331583478 -1.82441822262716
-4.75050442087492 -1.81849859097388
-4.73274552591506 -1.81257895932059
-4.7149866309552 -1.8066593276673
-4.69722773599533 -1.80073969601401
-4.67946884103547 -1.79482006436073
-4.66170994607561 -1.78890043270744
-4.64395105111575 -1.78298080105415
-4.62619215615589 -1.77706116940087
-4.60843326119603 -1.77114153774758
-4.59067436623616 -1.76522190609429
-4.5729154712763 -1.759302274441
-4.55515657631644 -1.75338264278772
-4.53739768135658 -1.74746301113443
-4.51963878639672 -1.74154337948114
-4.50187989143686 -1.73562374782786
-4.484120996477 -1.72970411617457
-4.46636210151713 -1.72378448452128
-4.44860320655727 -1.71786485286799
-4.43084431159741 -1.71194522121471
-4.41308541663755 -1.70602558956142
-4.39532652167769 -1.70010595790813
-4.37756762671783 -1.69418632625485
-4.35980873175797 -1.68826669460156
-4.3420498367981 -1.68234706294827
-4.32429094183824 -1.67642743129498
-4.30653204687838 -1.6705077996417
-4.28877315191852 -1.66458816798841
-4.27101425695866 -1.65866853633512
-4.2532553619988 -1.65274890468184
-4.23549646703893 -1.64682927302855
-4.21773757207907 -1.64090964137526
-4.19997867711921 -1.63499000972197
-4.18221978215935 -1.62907037806869
-4.16446088719949 -1.6231507464154
-4.14670199223963 -1.61723111476211
-4.12894309727977 -1.61131148310883
-4.1111842023199 -1.60539185145554
-4.09342530736004 -1.59947221980225
-4.07566641240018 -1.59355258814896
-4.05790751744032 -1.58763295649568
-4.04014862248046 -1.58171332484239
-4.0223897275206 -1.5757936931891
-4.00463083256073 -1.56987406153581
-3.98687193760087 -1.56395442988253
-3.96911304264101 -1.55803479822924
-3.95135414768115 -1.55211516657595
-3.93359525272129 -1.54619553492267
-3.91583635776143 -1.54027590326938
-3.89807746280156 -1.53435627161609
-3.8803185678417 -1.5284366399628
-3.86255967288184 -1.52251700830952
-3.84480077792198 -1.51659737665623
-3.82704188296212 -1.51067774500294
-3.80928298800226 -1.50475811334966
-3.7915240930424 -1.49883848169637
-3.77376519808253 -1.49291885004308
-3.75600630312267 -1.48699921838979
-3.73824740816281 -1.48107958673651
-3.72048851320295 -1.47515995508322
-3.70272961824309 -1.46924032342993
-3.68497072328323 -1.46332069177665
-3.66721182832337 -1.45740106012336
-3.6494529333635 -1.45148142847007
-3.63169403840364 -1.44556179681678
-3.61393514344378 -1.4396421651635
-3.59617624848392 -1.43372253351021
-3.57841735352406 -1.42780290185692
-3.5606584585642 -1.42188327020364
-3.54289956360433 -1.41596363855035
-3.52514066864447 -1.41004400689706
-3.50738177368461 -1.40412437524377
-3.48962287872475 -1.39820474359049
-3.47186398376489 -1.3922851119372
-3.45410508880503 -1.38636548028391
-3.43634619384517 -1.38044584863063
-3.4185872988853 -1.37452621697734
-3.40082840392544 -1.36860658532405
-3.38306950896558 -1.36268695367076
-3.36531061400572 -1.35676732201748
-3.34755171904586 -1.35084769036419
-3.329792824086 -1.3449280587109
-3.31203392912613 -1.33900842705762
-3.29427503416627 -1.33308879540433
-3.27651613920641 -1.32716916375104
-3.25875724424655 -1.32124953209775
-3.24099834928669 -1.31532990044447
-3.22323945432683 -1.30941026879118
-3.20548055936697 -1.30349063713789
-3.1877216644071 -1.2975710054846
-3.16996276944724 -1.29165137383132
-3.15220387448738 -1.28573174217803
-3.13444497952752 -1.27981211052474
-3.11668608456766 -1.27389247887146
-3.0989271896078 -1.26797284721817
-3.08116829464793 -1.26205321556488
-3.06340939968807 -1.25613358391159
-3.04565050472821 -1.25021395225831
-3.02789160976835 -1.24429432060502
-3.01013271480849 -1.23837468895173
-2.99237381984863 -1.23245505729845
-2.97461492488877 -1.22653542564516
-2.9568560299289 -1.22061579399187
-2.93909713496904 -1.21469616233858
-2.92133824000918 -1.2087765306853
-2.90357934504932 -1.20285689903201
-2.88582045008946 -1.19693726737872
-2.8680615551296 -1.19101763572544
-2.85030266016973 -1.18509800407215
-2.83254376520987 -1.17917837241886
-2.81478487025001 -1.17325874076557
-2.79702597529015 -1.16733910911229
-2.77926708033029 -1.161419477459
-2.76150818537043 -1.15549984580571
-2.74374929041057 -1.14958021415243
-2.7259903954507 -1.14366058249914
-2.70823150049084 -1.13774095084585
-2.69047260553098 -1.13182131919256
-2.67271371057112 -1.12590168753928
-2.65495481561126 -1.11998205588599
-2.6371959206514 -1.1140624242327
-2.61943702569153 -1.10814279257942
-2.60167813073167 -1.10222316092613
-2.58391923577181 -1.09630352927284
-2.56616034081195 -1.09038389761955
-2.54840144585209 -1.08446426596627
-2.53064255089223 -1.07854463431298
-2.51288365593237 -1.07262500265969
-2.4951247609725 -1.0667053710064
-2.47736586601264 -1.06078573935312
-2.45960697105278 -1.05486610769983
-2.44184807609292 -1.04894647604654
-2.42408918113306 -1.04302684439326
-2.4063302861732 -1.03710721273997
-2.38857139121333 -1.03118758108668
-2.37081249625347 -1.02526794943339
-2.35305360129361 -1.01934831778011
-2.33529470633375 -1.01342868612682
-2.31753581137389 -1.00750905447353
-2.29977691641403 -1.00158942282025
-2.28201802145417 -0.995669791166959
-2.2642591264943 -0.989750159513672
-2.24650023153444 -0.983830527860384
-2.22874133657458 -0.977910896207097
-2.21098244161472 -0.97199126455381
-2.19322354665486 -0.966071632900523
-2.175464651695 -0.960152001247236
-2.15770575673513 -0.954232369593948
-2.13994686177527 -0.948312737940661
-2.12218796681541 -0.942393106287374
-2.10442907185555 -0.936473474634087
-2.08667017689569 -0.9305538429808
-2.06891128193583 -0.924634211327513
-2.05115238697597 -0.918714579674225
-2.0333934920161 -0.912794948020938
-2.01563459705624 -0.906875316367651
-1.99787570209638 -0.900955684714364
-1.98011680713652 -0.895036053061077
-1.96235791217666 -0.889116421407789
-1.9445990172168 -0.883196789754502
-1.92684012225693 -0.877277158101215
-1.90908122729707 -0.871357526447928
-1.89132233233721 -0.865437894794641
-1.87356343737735 -0.859518263141354
-1.85580454241749 -0.853598631488066
-1.83804564745763 -0.847678999834779
-1.82028675249777 -0.841759368181492
-1.8025278575379 -0.835839736528205
-1.78476896257804 -0.829920104874918
-1.76701006761818 -0.824000473221631
};
\addlegendentry{slope $1/3$}
\addplot [line width=0.72pt, black, opacity=0.6, dotted, forget plot]
table {%
-3 -3.65114519157863
-3 0.236276915135056
};

\nextgroupplot[
legend cell align={left},
legend style={
  fill opacity=0.8,
  draw opacity=1,
  text opacity=1,
  at={(0.03,0.97)},
  anchor=north west,
  draw=none
},
tick align=outside,
tick pos=left,
title={\(\displaystyle d=15\)},
x grid style={darkgray176},
xlabel={\(\displaystyle \log_{d}p\)},
xmajorgrids,
xmin=-4.64601598228545, xmax=-1.59490098855315,
xtick style={color=black},
y grid style={darkgray176},
ylabel={},
ymajorgrids,
ymin=-2.70391307631532, ymax=0.347201917416982,
ytick style={color=black}
]
\addplot [line width=0.72pt, darkcyan0114178, mark=*, mark size=3, mark options={solid}, only marks]
table {%
-4.5073289371158 -2.56522603114567
-4.39175640071474 -2.45055412257478
-4.27618386272019 -2.33621474794936
-4.16061132963703 -2.22231499735366
-4.04503877271605 -2.1086185355013
-3.9294662355578 -1.99573932618764
-3.81389371708546 -1.88360348682445
-3.69832118341301 -1.77246216703164
-3.58274863322009 -1.66246793489373
-3.46717610197897 -1.55405068321802
-3.3516035748017 -1.44747407491583
-3.23603102792066 -1.34250634166802
-3.12045848311824 -1.24010689433383
-3.00488595223149 -1.1390668152924
-2.88931340645662 -1.04067033916162
-2.77374087180854 -0.943870774855888
-2.65816834370882 -0.850107885073805
-2.54259582364756 -0.75838877074941
-2.42702327143528 -0.668772517419401
-2.31145072713623 -0.581739773794771
-2.19587819487777 -0.495659497017167
-2.08030565612512 -0.412432118508558
-1.96473312281689 -0.332267200789399
-1.8491605805405 -0.255581512767186
-1.7335880337228 -0.198302926171486
};
%\addlegendentry{experimental values}
\addplot [thick, chocolate213940, dashed]
table {%
-4.5073289371158 -2.56522603114567
-4.49339054061634 -2.55128763464621
-4.47945214411688 -2.53734923814675
-4.46551374761741 -2.52341084164729
-4.45157535111795 -2.50947244514782
-4.43763695461849 -2.49553404864836
-4.42369855811903 -2.4815956521489
-4.40976016161956 -2.46765725564944
-4.3958217651201 -2.45371885914997
-4.38188336862064 -2.43978046265051
-4.36794497212118 -2.42584206615105
-4.35400657562171 -2.41190366965159
-4.34006817912225 -2.39796527315213
-4.32612978262279 -2.38402687665266
-4.31219138612333 -2.3700884801532
-4.29825298962387 -2.35615008365374
-4.2843145931244 -2.34221168715428
-4.27037619662494 -2.32827329065481
-4.25643780012548 -2.31433489415535
-4.24249940362602 -2.30039649765589
-4.22856100712655 -2.28645810115643
-4.21462261062709 -2.27251970465696
-4.20068421412763 -2.2585813081575
-4.18674581762817 -2.24464291165804
-4.1728074211287 -2.23070451515858
-4.15886902462924 -2.21676611865911
-4.14493062812978 -2.20282772215965
-4.13099223163032 -2.18888932566019
-4.11705383513085 -2.17495092916073
-4.10311543863139 -2.16101253266126
-4.08917704213193 -2.1470741361618
-4.07523864563247 -2.13313573966234
-4.06130024913301 -2.11919734316288
-4.04736185263354 -2.10525894666342
-4.03342345613408 -2.09132055016395
-4.01948505963462 -2.07738215366449
-4.00554666313516 -2.06344375716503
-3.99160826663569 -2.04950536066557
-3.97766987013623 -2.0355669641661
-3.96373147363677 -2.02162856766664
-3.94979307713731 -2.00769017116718
-3.93585468063784 -1.99375177466772
-3.92191628413838 -1.97981337816825
-3.90797788763892 -1.96587498166879
-3.89403949113946 -1.95193658516933
-3.88010109464 -1.93799818866987
-3.86616269814053 -1.92405979217041
-3.85222430164107 -1.91012139567094
-3.83828590514161 -1.89618299917148
-3.82434750864215 -1.88224460267202
-3.81040911214268 -1.86830620617256
-3.79647071564322 -1.85436780967309
-3.78253231914376 -1.84042941317363
-3.7685939226443 -1.82649101667417
-3.75465552614483 -1.81255262017471
-3.74071712964537 -1.79861422367524
-3.72677873314591 -1.78467582717578
-3.71284033664645 -1.77073743067632
-3.69890194014698 -1.75679903417686
-3.68496354364752 -1.74286063767739
-3.67102514714806 -1.72892224117793
-3.6570867506486 -1.71498384467847
-3.64314835414914 -1.70104544817901
-3.62920995764967 -1.68710705167955
-3.61527156115021 -1.67316865518008
-3.60133316465075 -1.65923025868062
-3.58739476815129 -1.64529186218116
-3.57345637165182 -1.6313534656817
-3.55951797515236 -1.61741506918223
-3.5455795786529 -1.60347667268277
-3.53164118215344 -1.58953827618331
-3.51770278565397 -1.57559987968385
-3.50376438915451 -1.56166148318438
-3.48982599265505 -1.54772308668492
-3.47588759615559 -1.53378469018546
-3.46194919965613 -1.519846293686
-3.44801080315666 -1.50590789718654
-3.4340724066572 -1.49196950068707
-3.42013401015774 -1.47803110418761
-3.40619561365828 -1.46409270768815
-3.39225721715881 -1.45015431118869
-3.37831882065935 -1.43621591468922
-3.36438042415989 -1.42227751818976
-3.35044202766043 -1.4083391216903
-3.33650363116096 -1.39440072519084
-3.3225652346615 -1.38046232869137
-3.30862683816204 -1.36652393219191
-3.29468844166258 -1.35258553569245
-3.28075004516311 -1.33864713919299
-3.26681164866365 -1.32470874269352
-3.25287325216419 -1.31077034619406
-3.23893485566473 -1.2968319496946
-3.22499645916527 -1.28289355319514
-3.2110580626658 -1.26895515669568
-3.19711966616634 -1.25501676019621
-3.18318126966688 -1.24107836369675
-3.16924287316742 -1.22713996719729
-3.15530447666795 -1.21320157069783
-3.14136608016849 -1.19926317419836
-3.12742768366903 -1.1853247776989
-3.11348928716957 -1.17138638119944
-3.0995508906701 -1.15744798469998
-3.08561249417064 -1.14350958820051
-3.07167409767118 -1.12957119170105
-3.05773570117172 -1.11563279520159
-3.04379730467226 -1.10169439870213
-3.02985890817279 -1.08775600220267
-3.01592051167333 -1.0738176057032
-3.00198211517387 -1.05987920920374
-2.98804371867441 -1.04594081270428
-2.97410532217494 -1.03200241620482
-2.96016692567548 -1.01806401970535
-2.94622852917602 -1.00412562320589
-2.93229013267656 -0.990187226706429
-2.91835173617709 -0.976248830206966
-2.90441333967763 -0.962310433707504
-2.89047494317817 -0.948372037208042
-2.87653654667871 -0.93443364070858
-2.86259815017924 -0.920495244209117
-2.84865975367978 -0.906556847709655
-2.83472135718032 -0.892618451210192
-2.82078296068086 -0.87868005471073
-2.8068445641814 -0.864741658211268
-2.79290616768193 -0.850803261711806
-2.77896777118247 -0.836864865212343
-2.76502937468301 -0.822926468712881
-2.75109097818355 -0.808988072213419
-2.73715258168408 -0.795049675713956
-2.72321418518462 -0.781111279214494
-2.70927578868516 -0.767172882715031
-2.6953373921857 -0.753234486215569
-2.68139899568623 -0.739296089716107
-2.66746059918677 -0.725357693216645
-2.65352220268731 -0.711419296717182
-2.63958380618785 -0.69748090021772
-2.62564540968839 -0.683542503718257
-2.61170701318892 -0.669604107218795
-2.59776861668946 -0.655665710719333
-2.58383022019 -0.641727314219871
-2.56989182369054 -0.627788917720408
-2.55595342719107 -0.613850521220946
-2.54201503069161 -0.599912124721484
-2.52807663419215 -0.585973728222021
-2.51413823769269 -0.572035331722559
-2.50019984119322 -0.558096935223097
-2.48626144469376 -0.544158538723634
-2.4723230481943 -0.530220142224172
-2.45838465169484 -0.516281745724709
-2.44444625519537 -0.502343349225247
-2.43050785869591 -0.488404952725785
-2.41656946219645 -0.474466556226322
-2.40263106569699 -0.46052815972686
-2.38869266919753 -0.446589763227398
-2.37475427269806 -0.432651366727935
-2.3608158761986 -0.418712970228473
-2.34687747969914 -0.404774573729011
-2.33293908319968 -0.390836177229549
-2.31900068670021 -0.376897780730086
-2.30506229020075 -0.362959384230624
-2.29112389370129 -0.349020987731162
-2.27718549720183 -0.335082591231699
-2.26324710070236 -0.321144194732237
-2.2493087042029 -0.307205798232774
-2.23537030770344 -0.293267401733312
-2.22143191120398 -0.27932900523385
-2.20749351470451 -0.265390608734387
-2.19355511820505 -0.251452212234925
-2.17961672170559 -0.237513815735463
-2.16567832520613 -0.223575419236
-2.15173992870667 -0.209637022736538
-2.1378015322072 -0.195698626237076
-2.12386313570774 -0.181760229737614
-2.10992473920828 -0.167821833238151
-2.09598634270882 -0.153883436738689
-2.08204794620935 -0.139945040239227
-2.06810954970989 -0.126006643739764
-2.05417115321043 -0.112068247240302
-2.04023275671097 -0.0981298507408392
-2.0262943602115 -0.0841914542413771
-2.01235596371204 -0.0702530577419149
-1.99841756721258 -0.0563146612424523
-1.98447917071312 -0.0423762647429902
-1.97054077421366 -0.0284378682435276
-1.95660237771419 -0.0144994717440654
-1.94266398121473 -0.000561075244603249
-1.92872558471527 0.0133773212548594
-1.91478718821581 0.0273157177543215
-1.90084879171634 0.0412541142537841
-1.88691039521688 0.0551925107532463
-1.87297199871742 0.0691309072527084
-1.85903360221796 0.083069303752171
-1.84509520571849 0.0970077002516332
-1.83115680921903 0.110946096751095
-1.81721841271957 0.124884493250558
-1.80328001622011 0.13882288975002
-1.78934161972064 0.152761286249483
-1.77540322322118 0.166699682748945
-1.76146482672172 0.180638079248407
-1.74752643022226 0.19457647574787
-1.7335880337228 0.208514872247332
};
%\addlegendentry{slope $1$}
\addplot [thick, darkcyan0158115, dashed]
table {%
-4.5073289371158 -1.63988114358717
-4.49339054061634 -1.63523501142068
-4.47945214411688 -1.63058887925419
-4.46551374761741 -1.6259427470877
-4.45157535111795 -1.62129661492122
-4.43763695461849 -1.61665048275473
-4.42369855811903 -1.61200435058824
-4.40976016161956 -1.60735821842175
-4.3958217651201 -1.60271208625527
-4.38188336862064 -1.59806595408878
-4.36794497212118 -1.59341982192229
-4.35400657562171 -1.5887736897558
-4.34006817912225 -1.58412755758932
-4.32612978262279 -1.57948142542283
-4.31219138612333 -1.57483529325634
-4.29825298962387 -1.57018916108985
-4.2843145931244 -1.56554302892337
-4.27037619662494 -1.56089689675688
-4.25643780012548 -1.55625076459039
-4.24249940362602 -1.5516046324239
-4.22856100712655 -1.54695850025742
-4.21462261062709 -1.54231236809093
-4.20068421412763 -1.53766623592444
-4.18674581762817 -1.53302010375795
-4.1728074211287 -1.52837397159147
-4.15886902462924 -1.52372783942498
-4.14493062812978 -1.51908170725849
-4.13099223163032 -1.514435575092
-4.11705383513085 -1.50978944292552
-4.10311543863139 -1.50514331075903
-4.08917704213193 -1.50049717859254
-4.07523864563247 -1.49585104642606
-4.06130024913301 -1.49120491425957
-4.04736185263354 -1.48655878209308
-4.03342345613408 -1.48191264992659
-4.01948505963462 -1.47726651776011
-4.00554666313516 -1.47262038559362
-3.99160826663569 -1.46797425342713
-3.97766987013623 -1.46332812126064
-3.96373147363677 -1.45868198909416
-3.94979307713731 -1.45403585692767
-3.93585468063784 -1.44938972476118
-3.92191628413838 -1.44474359259469
-3.90797788763892 -1.44009746042821
-3.89403949113946 -1.43545132826172
-3.88010109464 -1.43080519609523
-3.86616269814053 -1.42615906392874
-3.85222430164107 -1.42151293176226
-3.83828590514161 -1.41686679959577
-3.82434750864215 -1.41222066742928
-3.81040911214268 -1.40757453526279
-3.79647071564322 -1.40292840309631
-3.78253231914376 -1.39828227092982
-3.7685939226443 -1.39363613876333
-3.75465552614483 -1.38899000659684
-3.74071712964537 -1.38434387443036
-3.72677873314591 -1.37969774226387
-3.71284033664645 -1.37505161009738
-3.69890194014698 -1.37040547793089
-3.68496354364752 -1.36575934576441
-3.67102514714806 -1.36111321359792
-3.6570867506486 -1.35646708143143
-3.64314835414914 -1.35182094926494
-3.62920995764967 -1.34717481709846
-3.61527156115021 -1.34252868493197
-3.60133316465075 -1.33788255276548
-3.58739476815129 -1.33323642059899
-3.57345637165182 -1.32859028843251
-3.55951797515236 -1.32394415626602
-3.5455795786529 -1.31929802409953
-3.53164118215344 -1.31465189193304
-3.51770278565397 -1.31000575976656
-3.50376438915451 -1.30535962760007
-3.48982599265505 -1.30071349543358
-3.47588759615559 -1.29606736326709
-3.46194919965613 -1.29142123110061
-3.44801080315666 -1.28677509893412
-3.4340724066572 -1.28212896676763
-3.42013401015774 -1.27748283460115
-3.40619561365828 -1.27283670243466
-3.39225721715881 -1.26819057026817
-3.37831882065935 -1.26354443810168
-3.36438042415989 -1.2588983059352
-3.35044202766043 -1.25425217376871
-3.33650363116096 -1.24960604160222
-3.3225652346615 -1.24495990943573
-3.30862683816204 -1.24031377726925
-3.29468844166258 -1.23566764510276
-3.28075004516311 -1.23102151293627
-3.26681164866365 -1.22637538076978
-3.25287325216419 -1.2217292486033
-3.23893485566473 -1.21708311643681
-3.22499645916527 -1.21243698427032
-3.2110580626658 -1.20779085210383
-3.19711966616634 -1.20314471993735
-3.18318126966688 -1.19849858777086
-3.16924287316742 -1.19385245560437
-3.15530447666795 -1.18920632343788
-3.14136608016849 -1.1845601912714
-3.12742768366903 -1.17991405910491
-3.11348928716957 -1.17526792693842
-3.0995508906701 -1.17062179477193
-3.08561249417064 -1.16597566260545
-3.07167409767118 -1.16132953043896
-3.05773570117172 -1.15668339827247
-3.04379730467226 -1.15203726610598
-3.02985890817279 -1.1473911339395
-3.01592051167333 -1.14274500177301
-3.00198211517387 -1.13809886960652
-2.98804371867441 -1.13345273744003
-2.97410532217494 -1.12880660527355
-2.96016692567548 -1.12416047310706
-2.94622852917602 -1.11951434094057
-2.93229013267656 -1.11486820877408
-2.91835173617709 -1.1102220766076
-2.90441333967763 -1.10557594444111
-2.89047494317817 -1.10092981227462
-2.87653654667871 -1.09628368010813
-2.86259815017924 -1.09163754794165
-2.84865975367978 -1.08699141577516
-2.83472135718032 -1.08234528360867
-2.82078296068086 -1.07769915144218
-2.8068445641814 -1.0730530192757
-2.79290616768193 -1.06840688710921
-2.77896777118247 -1.06376075494272
-2.76502937468301 -1.05911462277624
-2.75109097818355 -1.05446849060975
-2.73715258168408 -1.04982235844326
-2.72321418518462 -1.04517622627677
-2.70927578868516 -1.04053009411029
-2.6953373921857 -1.0358839619438
-2.68139899568623 -1.03123782977731
-2.66746059918677 -1.02659169761082
-2.65352220268731 -1.02194556544434
-2.63958380618785 -1.01729943327785
-2.62564540968839 -1.01265330111136
-2.61170701318892 -1.00800716894487
-2.59776861668946 -1.00336103677839
-2.58383022019 -0.998714904611898
-2.56989182369054 -0.994068772445411
-2.55595342719107 -0.989422640278924
-2.54201503069161 -0.984776508112436
-2.52807663419215 -0.980130375945949
-2.51413823769269 -0.975484243779461
-2.50019984119322 -0.970838111612974
-2.48626144469376 -0.966191979446486
-2.4723230481943 -0.961545847279999
-2.45838465169484 -0.956899715113511
-2.44444625519537 -0.952253582947024
-2.43050785869591 -0.947607450780537
-2.41656946219645 -0.942961318614049
-2.40263106569699 -0.938315186447562
-2.38869266919753 -0.933669054281074
-2.37475427269806 -0.929022922114587
-2.3608158761986 -0.924376789948099
-2.34687747969914 -0.919730657781612
-2.33293908319968 -0.915084525615125
-2.31900068670021 -0.910438393448637
-2.30506229020075 -0.90579226128215
-2.29112389370129 -0.901146129115662
-2.27718549720183 -0.896499996949175
-2.26324710070236 -0.891853864782687
-2.2493087042029 -0.8872077326162
-2.23537030770344 -0.882561600449712
-2.22143191120398 -0.877915468283225
-2.20749351470451 -0.873269336116737
-2.19355511820505 -0.86862320395025
-2.17961672170559 -0.863977071783762
-2.16567832520613 -0.859330939617275
-2.15173992870667 -0.854684807450788
-2.1378015322072 -0.8500386752843
-2.12386313570774 -0.845392543117813
-2.10992473920828 -0.840746410951325
-2.09598634270882 -0.836100278784838
-2.08204794620935 -0.83145414661835
-2.06810954970989 -0.826808014451863
-2.05417115321043 -0.822161882285376
-2.04023275671097 -0.817515750118888
-2.0262943602115 -0.812869617952401
-2.01235596371204 -0.808223485785913
-1.99841756721258 -0.803577353619426
-1.98447917071312 -0.798931221452938
-1.97054077421366 -0.794285089286451
-1.95660237771419 -0.789638957119963
-1.94266398121473 -0.784992824953476
-1.92872558471527 -0.780346692786989
-1.91478718821581 -0.775700560620501
-1.90084879171634 -0.771054428454014
-1.88691039521688 -0.766408296287526
-1.87297199871742 -0.761762164121039
-1.85903360221796 -0.757116031954551
-1.84509520571849 -0.752469899788064
-1.83115680921903 -0.747823767621577
-1.81721841271957 -0.743177635455089
-1.80328001622011 -0.738531503288602
-1.78934161972064 -0.733885371122114
-1.77540322322118 -0.729239238955627
-1.76146482672172 -0.724593106789139
-1.74752643022226 -0.719946974622652
-1.7335880337228 -0.715300842456164
};
%addlegendentry{slope $1/3$}
\addplot [line width=0.72pt, black, opacity=0.6, dotted, forget plot]
table {%
-3 -2.70391307631532
-3 0.347201917416982
};
\end{groupplot}

\end{tikzpicture}